\documentclass{article}
\pdfoutput=1
\usepackage{times}
\usepackage{afterpage}
\usepackage{graphicx}
\usepackage{subfigure}
\usepackage{natbib}
\usepackage{algorithm}
\usepackage{algorithmic}
\usepackage{hyperref}
\usepackage{amsmath,amssymb,amsthm}
\newtheorem{proposition}{Proposition}

\usepackage[accepted]{icml2015}
\icmltitlerunning{Weight Uncertainty in Neural Networks}

\newcommand{\compresslist}{%
        \setlength{\leftmargini}{0pt}
    \setlength{\itemsep}{1pt}%
        \setlength{\parskip}{0pt}%
        \setlength{\parsep}{0pt}%
    }

\begin{document} 

\twocolumn[
\icmltitle{Weight Uncertainty in Neural Networks}

\icmlauthor{Charles Blundell}{cblundell@google.com}
\icmlauthor{Julien Cornebise}{jucor@google.com}
\icmlauthor{Koray Kavukcuoglu}{korayk@google.com}
\icmlauthor{Daan Wierstra}{wierstra@google.com}
\icmladdress{Google DeepMind}
\icmlkeywords{deep learning, neural networks, Bayesian, probability}

\vskip 0.3in
]

\begin{abstract} 
We introduce a new, efficient, principled and backpropagation-compatible algorithm 
for learning a probability distribution on the weights of a
neural network, called \emph{Bayes by Backprop}.
It regularises the weights by minimising a compression cost, known as the variational free energy
or the expected lower bound on the marginal likelihood.
We show that this principled kind of regularisation yields comparable
performance to dropout on MNIST classification.
We then demonstrate how the learnt uncertainty in the weights can be used to
improve generalisation in non-linear regression problems, and how this weight
uncertainty can be used to drive the exploration-exploitation trade-off in
reinforcement learning.
\end{abstract} 

\section{Introduction}

Plain feedforward neural networks are prone to overfitting.
When applied to supervised or reinforcement learning problems these networks
are also often incapable of correctly assessing the uncertainty in the training
data and so make overly confident decisions about the correct class, prediction
or action.
We shall address both of these concerns by using variational Bayesian
learning to introduce uncertainty in the weights of the network.
We call our algorithm \emph{Bayes by Backprop}.
We suggest at least three motivations for introducing uncertainty on
the weights: 1) regularisation via a compression cost on the weights, 2) richer
representations and predictions from cheap model averaging, and 3) exploration
in simple reinforcement learning problems such as contextual bandits. 

Various regularisation schemes have been developed to prevent overfitting in neural networks
such as early stopping, weight decay, and dropout \citep{hinton_dropout_2012}.
In this work, we introduce an efficient, principled algorithm for regularisation built upon Bayesian inference on
the weights of the network \citep{mackay_practical_1992,buntine_bayesian_1991,mackay_probable_1995}.
This leads to a simple approximate learning
algorithm similar to backpropagation \citep{lecun_procedure_1985,
rumelhart1988learning}.
We shall demonstrate how this uncertainty can improve predictive
performance in regression problems by expressing uncertainty in regions with little or no data,
how this uncertainty can lead to more systematic exploration
than $\epsilon$-greedy in contextual bandit tasks.

All weights in our neural networks are represented by probability distributions
over possible values, rather than having a single fixed value as is the norm (see Figure~\ref{fig:schema}).
Learnt representations and computations must therefore be robust under
perturbation of the weights, but the amount of perturbation each weight exhibits is
also learnt in a way that coherently explains variability in the training data.
Thus instead of training a single network, the proposed method trains an ensemble of networks,
where each network has its weights drawn from a shared, learnt probability distribution.
Unlike other ensemble methods, our method typically only doubles the number of parameters
yet trains an infinite ensemble using unbiased Monte Carlo estimates of
the gradients.

In general, exact Bayesian inference on the weights of a neural network is intractable as
the number of parameters is very large and the functional form of a neural network does not lend itself
to exact integration.
Instead we take a variational approximation to exact Bayesian updates.
We build upon the work of \citet{graves_practical_2011}, who in turn built upon the work of \citet{hinton_keeping_1993}.
In contrast to this previous work, we show how the gradients of
\citet{graves_practical_2011} can be made unbiased and further how this method
can be used with non-Gaussian priors.
Consequently, Bayes by Backprop attains performance comparable to that of dropout \citep{hinton_dropout_2012}.
Our method is related to recent methods in deep, generative modelling
\citep{kingma_autoencoding_2014, rezende_stochastic_2014, gregor_deep_2014},
where variational inference has been applied to stochastic hidden units of an autoencoder.
Whilst the number of stochastic hidden units might be in the order of
thousands, the number of weights in a neural network is easily two orders of
magnitude larger, making the optimisation problem much larger scale.
Uncertainty in the hidden units allows the expression of uncertainty about a particular observation,
uncertainty in the weights is complementary in that it captures uncertainty
about which neural network is appropriate, leading to regularisation of the
weights and model averaging.

This uncertainty can be used to drive exploration in contextual bandit problems
using Thompson sampling \citep{thompson_likelihood_1933, chapelle_empirical_2011, agrawal_analysis_2012, may_optimistic_2012}.
Weights with greater uncertainty introduce more variability into the decisions made by the network,
leading naturally to exploration.
As more data are observed, the uncertainty can decrease, allowing the decisions made by the network
to become more deterministic as the environment is better understood.

\begin{figure}[t]
\begin{center}
\includegraphics[height=.15\textheight]{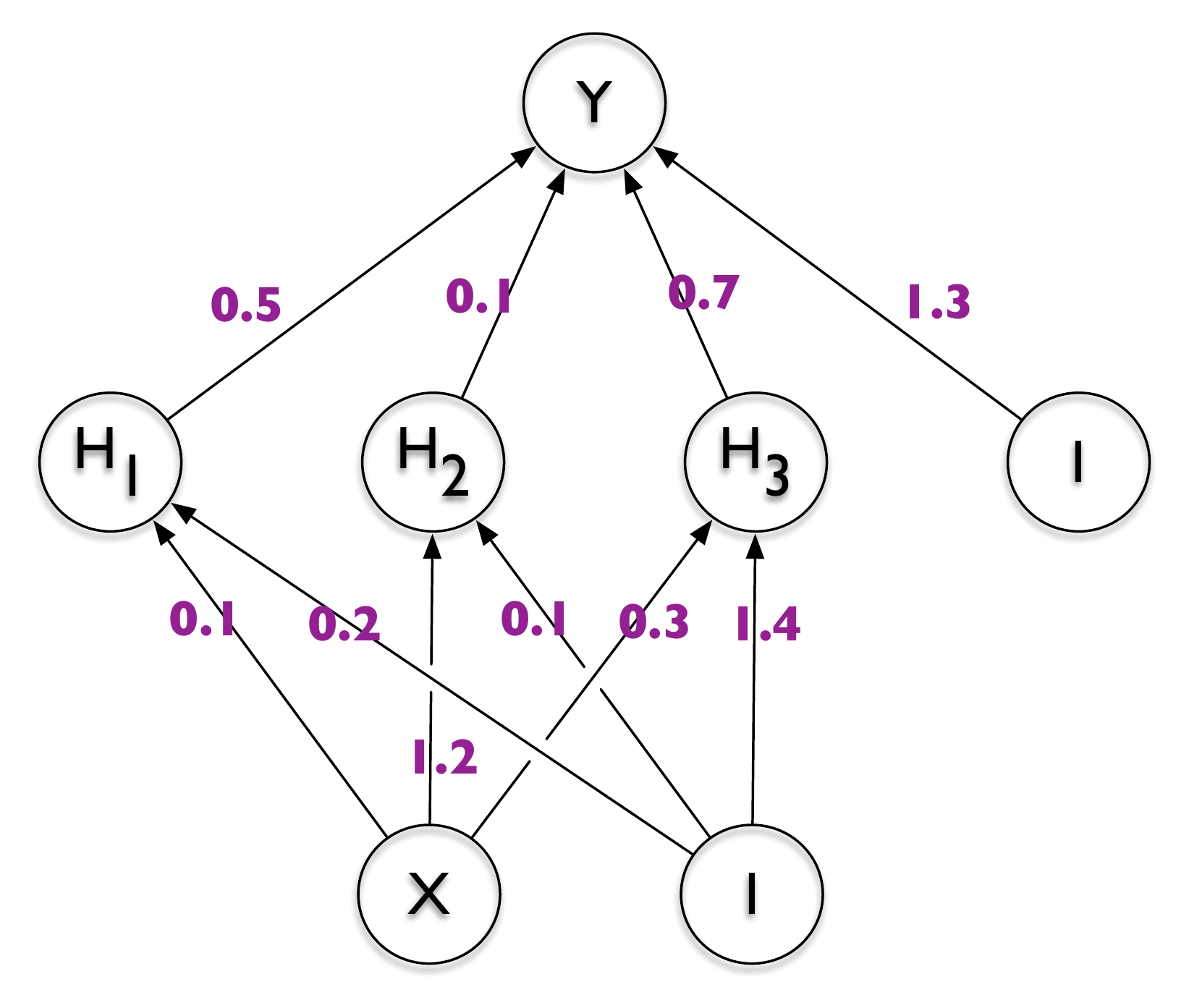}
\includegraphics[height=.15\textheight]{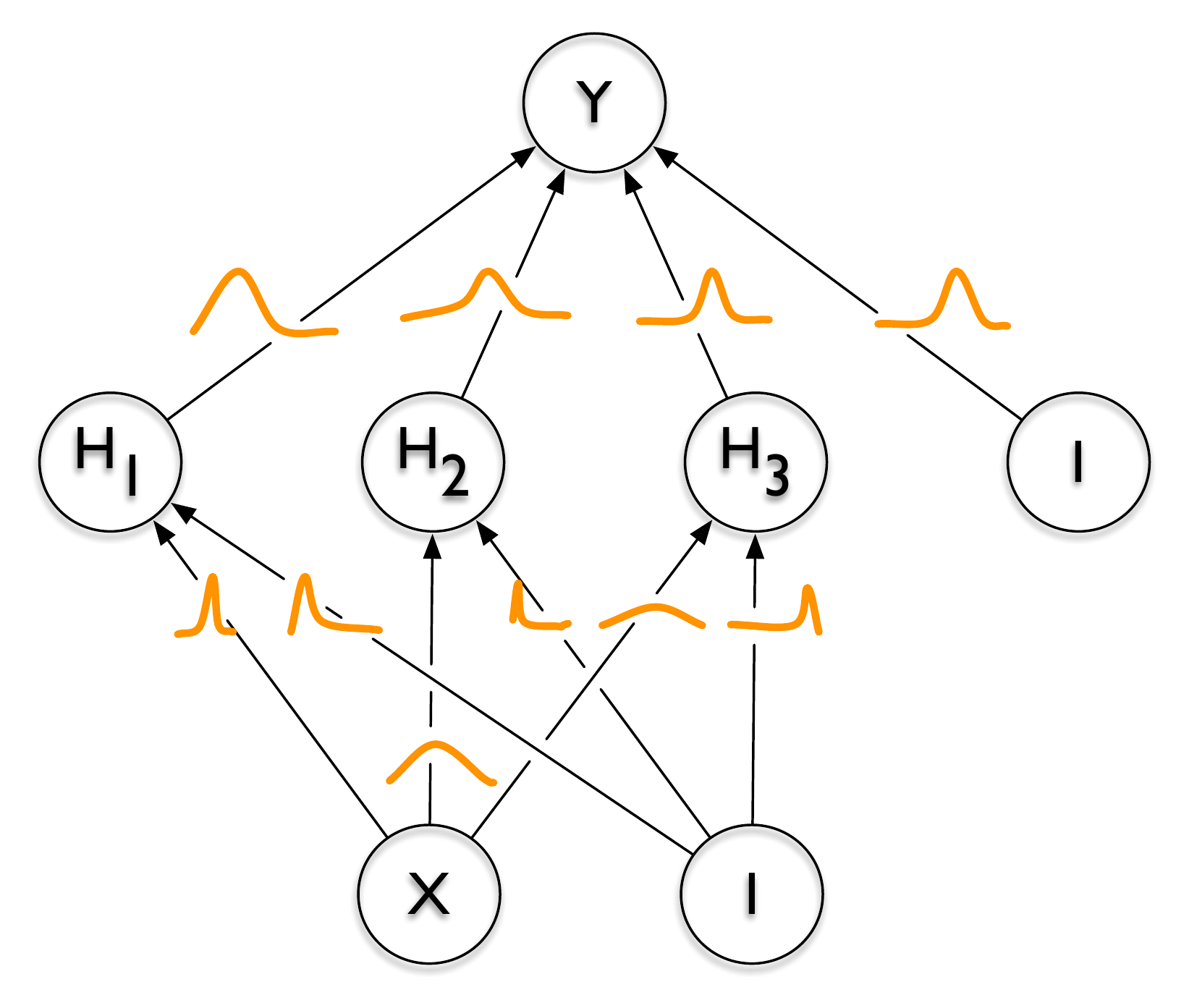}
\label{fig:schema}
\caption{Left: each weight has a fixed value, as provided by classical backpropagation. Right: each weight is assigned a distribution, as provided by Bayes by Backprop.}
\end{center}
\end{figure}

The remainder of the paper is organised as follows: Section~\ref{sec:point} introduces notation
and standard learning in neural networks, Section~\ref{sec:vb} describes variational Bayesian learning
for neural networks and our contributions, Section~\ref{sec:bandits} describes the application to contextual bandit problems,
whilst Section~\ref{sec:exper} contains empirical results on a classification, a regression and a bandit problem.
We conclude with a brief discussion in Section~\ref{sec:discuss}.

\section{Point Estimates of Neural Networks}
\label{sec:point}

We view a neural network as a probabilistic model $P(\mathbf{y}|
\mathbf{x}, \mathbf{w})$: given an input $\mathbf{x} \in \mathbb{R}^p$ a neural
network assigns a probability to each possible output $\mathbf{y} \in \mathcal{Y}$,
using the set of parameters or weights $\mathbf{w}$.
For classification, $\mathcal{Y}$ is a set of classes and
$P(\mathbf{y}|\mathbf{x}, \mathbf{w})$ is a categorical
distribution -- this corresponds to the cross-entropy or softmax loss,
when the parameters of the categorical distribution are passed through the
exponential function then re-normalised.
For regression $\mathcal{Y}$ is $\mathbb{R}$ and
$P(\mathbf{y}|\mathbf{x}, \mathbf{w})$ is a Gaussian distribution -- this
corresponds to a squared loss.

Inputs $\mathbf{x}$ are mapped onto the parameters of a distribution on
$\mathcal{Y}$ by several successive layers of linear transformation
(given by $\mathbf{w}$) interleaved with element-wise non-linear transforms.

The weights can be learnt by maximum likelihood estimation (MLE): given a set of training examples
$\mathcal{D} = (\mathbf{x}_i, \mathbf{y}_i)_i$, the MLE weights $\mathbf{w}^\text{MLE}$ are given by:
\begin{align*}
\mathbf{w}^\text{MLE}
&= \arg\max_\mathbf{w}
\log P(\mathcal{D}|\mathbf{w}) \\
&= \arg\max_\mathbf{w}
 \sum_i \log P(\mathbf{y}_i | \mathbf{x}_i, \mathbf{w})
.
\end{align*}
This is typically achieved by gradient descent (e.g., backpropagation), where
we assume that $\log P(\mathcal{D}|\mathbf{w})$ is differentiable in
$\mathbf{w}$.

Regularisation can be introduced by placing a prior upon the weights $\mathbf{w}$ and finding the
maximum a posteriori (MAP) weights $\mathbf{w}^\text{MAP}$:
\begin{align*}
\mathbf{w}^\text{MAP}
&= \arg\max_\mathbf{w}
\log P(\mathbf{w}|\mathcal{D}) \\
&= \arg\max_\mathbf{w}
\log P(\mathcal{D}|\mathbf{w}) + \log P(\mathbf{w})
.
\end{align*}
If $\mathbf{w}$ are given a Gaussian prior, this yields L2 regularisation (or weight decay).
If $\mathbf{w}$ are given a Laplace prior, then L1 regularisation is recovered.

\section{Being Bayesian by Backpropagation}
\label{sec:vb}

Bayesian inference for neural networks calculates the posterior distribution of the weights
given the training data, $P(\mathbf{w}|\mathcal{D})$.
This distribution answers predictive queries about unseen data by taking
expectations:
the predictive distribution of an unknown label $\hat{\mathbf{y}}$ of a test data item
$\hat{\mathbf{x}}$, is given by
$P(\hat{\mathbf{y}}|\hat{\mathbf{x}}) =
\mathbb{E}_{P(\mathbf{w}|\mathcal{D})}[P(\hat{\mathbf{y}}|\hat{\mathbf{x}},\mathbf{w})]$.
Each possible configuration of the weights, weighted according to the posterior
distribution, makes a prediction about the unknown label given the test data
item $\hat{\mathbf{x}}$.
Thus taking an expectation under the posterior distribution on weights is
equivalent to using an ensemble of an uncountably infinite number of neural
networks.
Unfortunately, this is intractable for neural networks of any practical size.

Previously \citet{hinton_keeping_1993} and \citet{graves_practical_2011}
suggested finding a variational approximation to the Bayesian posterior
distribution on the weights.
Variational learning finds the parameters $\theta$ of a distribution on the
weights $q(\mathbf{w}|\theta)$ that minimises the Kullback-Leibler (KL) divergence with the true
Bayesian posterior on the weights:
\begin{align*}
\theta^\star &= \arg\min_\theta
\text{KL}[q(\mathbf{w}|\theta)||P(\mathbf{w}|\mathcal{D})] \\
&= \arg\min_\theta
\int
q(\mathbf{w}|\theta)
\log
\frac{q(\mathbf{w}|\theta)}
     {P(\mathbf{w}) P(\mathcal{D}|\mathbf{w})}
\textrm{d}\mathbf{w} \\
&= \arg\min_\theta
\text{KL}
\left[
q(\mathbf{w}|\theta)
\mid \mid
P(\mathbf{w})
\right]
-
\mathbb{E}_{q(\mathbf{w}|\theta)}
\left[
\log
P(\mathcal{D}|\mathbf{w})
\right]
.
\end{align*}
The resulting cost function is variously known as the variational free energy
\citep{neal_view_1998,yedidia_generalized_2000,friston_variational_2007} or
the expected lower bound \citep{saul_mean_1996, neal_view_1998, jaakkola_bayesian_2000}.
For simplicity we shall denote it as
\begin{multline}
\label{eq:cost}
\mathcal{F}(\mathcal{D}, \theta) =
\text{KL}
\left[
q(\mathbf{w}|\theta)
\mid \mid
P(\mathbf{w})
\right] \\
-
\mathbb{E}_{q(\mathbf{w}|\theta)}
\left[
\log
P(\mathcal{D}|\mathbf{w})
\right]
.
\end{multline}
The cost function of \eqref{eq:cost} is a sum of a data-dependent part,
which we shall refer to as the likelihood cost, and a
prior-dependent part, which we shall refer to as the complexity cost.
The cost function embodies a trade-off between satisfying the complexity of the
data $\mathcal{D}$ and satisfying the simplicity prior $P(\mathbf{w})$.
\eqref{eq:cost} is also readily given an information theoretic interpretation
as a minimum description length cost \citep{hinton_keeping_1993,graves_practical_2011}.
Exactly minimising this cost na\"ively is computationally prohibitive.
Instead gradient descent and various approximations are used.

\subsection{Unbiased Monte Carlo gradients}
\label{sec:mc}

Under certain conditions, the derivative of an expectation can be
expressed as the expectation of a derivative:
\begin{proposition}
\label{prop:reparam}
Let $\mathbf{\epsilon}$ be a random variable having a probability density given by
$q(\mathbf{\epsilon})$ and let $\mathbf{w} = t(\theta, \mathbf{\epsilon})$
where $t(\theta,\mathbf{\epsilon})$ is a deterministic function.
Suppose further that the marginal probability density of $\mathbf{w}$,
$q(\mathbf{w}|\theta)$, is such that $q(\mathbf{\epsilon}) \textrm{d}\mathbf{\epsilon} = 
q(\mathbf{w}|\theta) \textrm{d}\mathbf{w}$. 
Then for a function $f$ with derivatives in $\mathbf{w}$:
\begin{align*}
\frac{\partial}{\partial \theta}
\mathbb{E}_{q(\mathbf{w}|\theta)}[f(\mathbf{w}, \theta)]
&=
\mathbb{E}_{q(\mathbf{\epsilon})}\left[
\frac{\partial f(\mathbf{w}, \theta)}{\partial \mathbf{w}}
\frac{\partial \mathbf{w}}{\partial \theta}
+
\frac{\partial f(\mathbf{w}, \theta)}{\partial \theta}
\right]
.
\end{align*}
\end{proposition}
\begin{proof}
\begin{align*}
\frac{\partial}{\partial \theta}
\mathbb{E}_{q(\mathbf{w}|\theta)}[f(\mathbf{w},\theta)]
&=
\frac{\partial}{\partial \theta}
\int
f(\mathbf{w},\theta)
q(\mathbf{w}|\theta)
\textrm{d}\mathbf{w} \\
&=
\frac{\partial}{\partial \theta}
\int
f(\mathbf{w},\theta)
q(\epsilon)
\textrm{d}\epsilon \\
&=
\mathbb{E}_{q(\epsilon)}\left[
\frac{\partial f(\mathbf{w}, \theta)}{\partial \mathbf{w}}
\frac{\partial \mathbf{w}}{\partial \theta}
+
\frac{\partial f(\mathbf{w}, \theta)}{\partial \theta}
\right]
\end{align*}
\end{proof}
The deterministic function $t(\theta,\epsilon)$ transforms a sample of parameter-free
noise $\epsilon$ and the variational posterior parameters $\theta$ into a sample from 
the variational posterior.
Below we shall see how this transform works in practice for the Gaussian case.

We apply Proposition~\ref{prop:reparam} to the optimisation problem in
\eqref{eq:cost}: let $f(\mathbf{w}, \theta) = 
\log q(\mathbf{w}|\theta) - \log P(\mathbf{w}) P(\mathcal{D}|\mathbf{w})$.
Using Monte Carlo sampling to evaluate the expectations,
a backpropagation-like \citep{lecun_procedure_1985,rumelhart1988learning}
algorithm is obtained for variational Bayesian inference in neural networks --
Bayes by Backprop -- which uses unbiased estimates of gradients of the cost in
\eqref{eq:cost} to learn a distribution over the weights of a neural network.

Proposition~\ref{prop:reparam} is a generalisation of the Gaussian re-parameterisation trick
\citep{opper_variational_2009,kingma_autoencoding_2014,rezende_stochastic_2014} used for latent
variable models, applied to Bayesian learning of neural networks.
Our work differs from this previous work in several significant ways.
Bayes by Backprop operates on weights (of which there are a great many), whilst most previous work
applies this method to learning distributions on stochastic hidden units (of which there are far fewer than
the number of weights).
\citet{titsias2014doubly} considered a large-scale logistic regression task.
Unlike previous work, we do not use the closed form of the complexity cost (or
entropic part): not requiring a closed form of the complexity cost allows many
more combinations of prior and variational posterior families.
Indeed this scheme is also simple to implement and allows prior/posterior
combinations to be interchanged.
We approximate the exact cost \eqref{eq:cost} as:
\begin{multline}
\label{eq:approxcost}
\mathcal{F}(\mathcal{D}, \theta) \approx
\sum_{i=1}^n
\log q(\mathbf{w}^{(i)}|\theta)
-
\log P(\mathbf{w}^{(i)}) \\
-
\log
P(\mathcal{D}|\mathbf{w}^{(i)})
\end{multline}
where $\mathbf{w}^{(i)}$ denotes the $i$th Monte Carlo sample drawn from the variational
posterior $q(\mathbf{w}^{(i)}|\theta)$.
Note that every term of this approximate cost depends upon the \emph{particular} 
weights drawn from the variational posterior: this is an instance of a variance reduction
technique known as common random numbers \citep{mcbook}.
In previous work, where a closed form complexity cost or closed form entropy
term are used, part of the cost is sensitive to particular draws from
the posterior, whilst the closed form part is oblivious.
Since each additive term in the approximate cost in \eqref{eq:approxcost}
uses the same weight samples, the gradients of \eqref{eq:approxcost} are only
affected by the parts of the posterior distribution characterised by the weight
samples.
In practice, we did not find this to perform better than using a closed form KL
(where it could be computed), but we did not find it to perform worse.
In our experiments, we found that a prior without an easy-to-compute closed
form complexity cost performed best.

\subsection{Gaussian variational posterior}
Suppose that the variational posterior is a diagonal Gaussian distribution,
then a sample of the weights $\mathbf{w}$ can be obtained by
sampling a unit Gaussian, shifting it by a mean $\mu$ and scaling by a standard deviation
$\sigma$.
We parameterise the standard deviation pointwise as $\sigma = \log(1+\exp(\rho))$ and 
so $\sigma$ is always non-negative.
The variational posterior parameters are $\theta = (\mu, \rho)$.
Thus the transform from a sample of parameter-free noise and the variational posterior
parameters that yields a posterior sample of the weights $\textbf{w}$ is:
$\textbf{w} = t(\theta, \epsilon) = \mu + \log(1+\exp(\rho))\circ\epsilon$
where $\circ$ is pointwise multiplication.
Each step of optimisation proceeds as follows:
\begin{enumerate}
\compresslist%
\item Sample $\mathbf{\epsilon} \sim \mathcal{N}(0, I)$.
\item Let $\mathbf{w} = \mathbf{\mu} + \log(1+\exp(\rho))\circ\mathbf{\epsilon}$.
\item Let $\theta = (\mathbf{\mu}, \mathbf{\rho})$.
\item Let $f(\mathbf{w}, \theta) = \log q(\mathbf{w}|\theta) - \log P(\mathbf{w})P(\mathcal{D}|\mathbf{w})$.
\item Calculate the gradient with respect to the mean
\begin{align}
\Delta_\mu &=
\frac{\partial f(\mathbf{w}, \theta)}{\partial \mathbf{w}}
+
\frac{\partial f(\mathbf{w}, \theta)}{\partial \mu}
.
\end{align}
\item Calculate the gradient with respect to the standard deviation parameter $\rho$
\begin{align}
\Delta_\rho &=
\frac{\partial f(\mathbf{w}, \theta)}{\partial \mathbf{w}}
\frac{\mathbf{\epsilon}}
     {1 + \exp(-\rho)}
+
\frac{\partial f(\mathbf{w}, \theta)}{\partial \rho}
.
\end{align}
\item Update the variational parameters:
\begin{align}
\mathbf{\mu} &\leftarrow \mathbf{\mu} - \alpha\Delta_\mu
\\
\mathbf{\rho} &\leftarrow \mathbf{\rho} - \alpha\Delta_\rho
.
\end{align}
\end{enumerate}
Note that the
$\frac{\partial f(\mathbf{w}, \theta)}{\partial \mathbf{w}}$
term of the gradients for the mean and standard deviation are shared and
are exactly the gradients found by the usual backpropagation algorithm
on a neural network.
Thus, remarkably, to learn both the mean and the standard deviation we must
simply calculate the usual gradients found by backpropagation, and then scale
and shift them as above.

\subsection{Scale mixture prior}

Having liberated our algorithm from the confines of Gaussian priors and posteriors,
we propose a simple scale mixture prior combined with a diagonal Gaussian posterior.
The diagonal Gaussian posterior is largely free from numerical issues,
and two degrees of freedom per weight only increases the number of parameters
to optimise by a factor of two, whilst giving each weight its own quantity of
uncertainty.

We pick a fixed-form prior and do not adjust its hyperparameters during training,
instead picking the them by cross-validation where possible.
Empirically we found optimising the parameters of a prior $P(\textbf{w})$ (by
taking derivatives of \eqref{eq:cost}) to not be useful, and yield worse
results.
\citet{graves_practical_2011} and \citet{titsias2014doubly} propose closed form
updates of the prior hyperparameters.
Changing the prior based upon the data that it is meant to regularise is known
as empirical Bayes and there is much debate as to its validity
\citep{gelman_objections_2008}.
A reason why it fails for Bayes by Backprop is as follows: it can be easier to
change the prior parameters (of which there are few) than it is to change the
posterior parameters (of which there are many) and so very quickly the prior
parameters try to capture the empirical distribution of the weights at the
beginning of learning.
Thus the prior learns to fit poor initial parameters quickly, and makes the cost in
\eqref{eq:cost} less willing to move away from poor initial parameters.
This can yield slow convergence, introduce strange local minima and result in
poor performance.

We propose using a scale mixture of two Gaussian densities as the prior.
Each density is zero mean, but differing variances:
\begin{align}
P(\mathbf{w})
&= \prod_j \pi \mathcal{N}(\mathbf{w}_j|0, \sigma^2_1) + (1-\pi) \mathcal{N}(\mathbf{w}_j|0, \sigma^2_2)
,
\end{align}
where $\mathbf{w}_j$ is the $j$th weight of the network, $\mathcal{N}(x|\mu,\sigma^2)$ is the Gaussian
density evaluated at $x$ with mean $\mu$ and variance $\sigma^2$ and
$\sigma_1^2$ and $\sigma_2^2$ are the variances of the mixture components.
The first mixture component of the prior is given a larger variance
than the second, $\sigma_1 > \sigma_2$, providing a heavier tail in the
prior density than a plain Gaussian prior.
The second mixture component has a small variance $\sigma_2 \ll 1$
causing many of the weights to \emph{a priori} tightly concentrate around zero.
Our prior resembles a spike-and-slab prior
\citep{mitchell1988bayesian,george1993variable,chipman1996bayesian}, where
instead all the prior parameters are shared among all the weights.
This makes the prior more amenable to use during optimisation by stochastic
gradient descent and avoids the need for prior parameter optimisation based upon
training data.

\subsection{Minibatches and KL re-weighting}
As several authors have noted, the cost in \eqref{eq:cost} is amenable to minibatch optimisation, 
often used with neural networks:
for each epoch of optimisation the training data $\mathcal{D}$ is
randomly split into a partition of $M$ equally-sized subsets, $\mathcal{D}_1, \mathcal{D}_2, \dots, \mathcal{D}_M$.
Each gradient is averaged over all elements in one of these minibatches; a trade-off between
a fully batched gradient descent and a fully stochastic gradient descent.
\citet{graves_practical_2011} proposes minimising the minibatch cost for minibatch $i=1,2,\dots,M$:
\begin{multline}
\mathcal{F}^\text{EQ}_i(\mathcal{D}_i, \theta) =
\frac{1}{M}
\text{KL}
\left[
q(\mathbf{w}|\theta)
\mid \mid
P(\mathbf{w})
\right] \\
-
\mathbb{E}_{q(\mathbf{w}|\theta)}
\left[
\log
P(\mathcal{D}_i|\mathbf{w})
\right]
.
\end{multline}
This is equivalent to the cost in \eqref{eq:cost} since $\sum_i
\mathcal{F}^\text{EQ}_i(\mathcal{D}_i, \theta) = \mathcal{F}(\mathcal{D},
\theta)$.
There are many ways to weight the complexity cost relative to the likelihood cost on
each minibatch.
For example, if minibatches are partitioned uniformly at random, the KL cost can be
distributed non-uniformly among the minibatches at each epoch.
Let $\mathbf{\pi}\in[0,1]^M$ and $\sum_{i=1}^M \pi_i = 1$, and define:
\begin{multline}
\mathcal{F}^\pi_i(\mathcal{D}_i, \theta) =
\pi_i
\text{KL}
\left[
q(\mathbf{w}|\theta)
\mid \mid
P(\mathbf{w})
\right] \\
-
\mathbb{E}_{q(\mathbf{w}|\theta)}
\left[
\log
P(\mathcal{D}_i|\mathbf{w})
\right]
\end{multline}
Then $\mathbb{E}_M[\sum_{i=1}^M \mathcal{F}^\pi_i(\mathcal{D}_i, \theta)] = \mathcal{F}(\mathcal{D},\theta)$
where $\mathbb{E}_M$ denotes an expectation over the random partitioning of minibatches.
In particular, we found the scheme $\pi_i = \frac{2^{M-i}}{2^M-1}$ to work well:
the first few minibatches are heavily influenced by the
complexity cost, whilst the later minibatches are largely influenced by the
data.
At the beginning of learning this is particularly useful as for the first few minibatches
changes in the weights due to the data are slight and as more data are seen, data become
more influential and the prior less influential.

\section{Contextual Bandits}
\label{sec:bandits}

Contextual bandits are simple reinforcement learning problems without
persistent state \citep{li_contextual_bandit_2010, filippi_parametric_2010}.
At each step an agent is presented with a context $x$ and a choice of one of
$K$ possible actions $a$.
Different actions yield different unknown rewards $r$.
The agent must pick the action that yields the highest expected reward.
The context is assumed to be presented independent of any previous actions,
rewards or contexts.

An agent builds a model of the distribution of the rewards
conditioned upon the action and the context: $P(r|x,a, \textbf{w})$.
It then uses this model to pick its action.
Note, importantly, that an agent does not know what reward it could have received for an
action that it did not pick, a difficulty often known as ``the absence of counterfactual''.
As the agent's model $P(r|x,a,\textbf{w})$ is trained online, based upon the
actions chosen, unless exploratory actions are taken, the agent may perform
suboptimally.

\subsection{Thompson Sampling for Neural Networks}

As in Section~\ref{sec:point}, $P(r | x, a, \mathbf{w})$ can be modelled by a neural
network where $\mathbf{w}$ are the weights of the neural network.
However if this network is simply fit to observations and the action with the
highest expected reward taken at each time, the agent can under-explore, as it
may miss more rewarding actions.\footnote{
Interestingly, depending upon how $\mathbf{w}$ are initialised and the mean of
prior used during MAP inference, it is sometimes possible to obtain
another heuristic for the exploration-exploitation trade-off: optimism-under-uncertainty.
We leave this for future investigation.}

Thompson sampling \citep{thompson_likelihood_1933} is a popular means of picking an action that trades-off
between exploitation (picking the best known action) and exploration (picking
what might be a suboptimal arm to learn more).
Thompson sampling usually necessitates a Bayesian treatment of the model
parameters.
At each step, Thompson sampling draws a new set of parameters
and then picks the action relative to those parameters.
This can be seen as a kind of stochastic hypothesis testing: more
probable parameters are drawn more often and thus refuted or confirmed the
fastest.
More concretely Thompson sampling proceeds as follows:
\begin{enumerate}
\compresslist%
\item Sample a new set of parameters for the model.
\item Pick the action with the highest expected reward according to the sampled parameters.
\item Update the model. Go to 1.
\end{enumerate}
There is an increasing literature concerning the efficacy and justification of this means of
exploration \citep{chapelle_empirical_2011,may_optimistic_2012,kaufmann_thompson_2012,agrawal_analysis_2012,agrawal_further_2013}.
Thompson sampling is easily adapted to neural networks using the variational posterior
found in Section~\ref{sec:vb}:
\begin{enumerate}
\compresslist%
\item Sample weights from the variational posterior: $\mathbf{w} \sim q(\mathbf{w}|\theta)$.
\item Receive the context $x$.
\item Pick the action $a$ that minimises $\mathbb{E}_{P(r|x,a,\mathbf{w})}[r]$
\item Receive reward $r$.
\item Update variational parameters $\theta$ according to Section~\ref{sec:vb}. Go to 1.
\end{enumerate}
Note that it is possible, as mentioned in Section~\ref{sec:mc}, to decrease the variance of the 
gradient estimates, trading off for reduced exploration, by using more than one Monte Carlo sample,
using the corresponding networks as an ensemble and picking the action by minimising the average of 
the expectations.

Initially the variational posterior will be close to the prior, and actions will be picked
uniformly.
As the agent takes actions, the variational posterior will begin to converge,
and uncertainty on many parameters can decrease, and so action selection will
become more deterministic, focusing on the high expected reward actions
discovered so far.
It is known that variational methods under-estimate uncertainty
\citep{minka_family_2001,minka_divergence_2005,bishop_section_2006}
which could lead to under-exploration and premature convergence in practice,
but we did not find this in practice.

\section{Experiments}
\label{sec:exper}

We present some empirical evaluation of the methods proposed above: on MNIST classification,
on a non-linear regression task, and on a contextual bandits task.

\begin{table}
\small
\caption{\label{tab:mnist}Classification Error Rates on MNIST. $\star$ indicates result used
an ensemble of $5$ networks.}
\centering
\begin{tabular}{l@{ }|l@{ }|l@{ }|r}
    \rotatebox{0}{\textbf{Method}} & \rotatebox{90}{\textbf{\# Units/Layer}} & 
\rotatebox{90}{\textbf{\# Weights}} & \begin{tabular}{@{}c@{}}\textbf{Test}\\\textbf{Error}\end{tabular} \\
\hline
SGD, no regularisation {\tiny \citep{simard_best_2003}}  & 800 & 1.3m & $1.6\%$ \\
SGD, dropout {\tiny \citep{hinton_dropout_2012}} &      &       & {\tiny $\approx$} $1.3\%$ \\
SGD, dropconnect {\tiny \citep{wan2013regularization}} & 800 & 1.3m & $\mathbf{1.2\%}^\star$ \\
\hline
SGD & 400 & 500k & $1.83\%$ \\
 & 800 & 1.3m & $1.84\%$ \\
 & 1200 & 2.4m & $1.88\%$ \\
\hline
SGD, dropout & 400 & 500k & $1.51\%$ \\
 & 800 & 1.3m & $1.33\%$ \\
 & 1200 & 2.4m & $1.36\%$ \\
\hline
Bayes by Backprop, Gaussian & 400 & 500k & $1.82\%$ \\
 & 800 & 1.3m & $1.99\%$ \\
 & 1200 & 2.4m & $2.04\%$ \\
\hline
Bayes by Backprop, Scale mixture & 400 & 500k & $1.36\%$ \\
 & 800 & 1.3m & $1.34\%$ \\
 & 1200 & 2.4m & $\mathbf{1.32\%}$ \\
\end{tabular}
\end{table}

\subsection{Classification on MNIST}
\label{sec:exper:mnist}
We trained networks of various sizes on the MNIST digits dataset \citep{lecun_mnist_1998},
consisting of 60,000 training and 10,000 testing pixel images of size 28 by 28.
Each image is labelled with its corresponding number (between zero
and nine, inclusive).
We preprocessed the pixels by dividing values by $126$.
Many methods have been proposed to improve results on MNIST: generative pre-training,
convolutions, distortions, etc.
Here we shall focus on improving the performance of an ordinary feedforward neural network
\emph{without} using any of these methods.
We used a network of two hidden layers of rectified linear units \citep{nair_rectified_2010,glorot_deep_2011}, and
a softmax output layer with 10 units, one for each possible label.

According to \citet{hinton_dropout_2012}, the best published feedforward neural network
classification result on MNIST (excluding those using data set augmentation,
convolutions, etc.) is $1.6\%$ \citep{simard_best_2003}, whilst dropout with an
L2 regulariser attains errors around $1.3\%$.
Results from Bayes by Backprop are shown in Table~\ref{tab:mnist}, for various sized networks,
using either a Gaussian or Gaussian scale mixture prior.
Performance is comparable to that of dropout, perhaps slightly better, as also see on Figure~\ref{fig:mnist}.
Note that we trained on 50,000 digits and used 10,000 digits as a validation set,
whilst \citet{hinton_dropout_2012} trained on 60,000 digits and did not use a validation set.
We used the validation set to pick the best hyperparameters (learning rate, number of gradients to average)
and so we also repeated this protocol for dropout and SGD (Stochastic Gradient Descent on the MLE objective in
Section~\ref{sec:point}).
We considered learning rates of $10^{-3}$, $10^{-4}$ and $10^{-5}$ with minibatches of size 128.
For Bayes by Backprop, we averaged over either 1, 2, 5, or 10 samples and
considered $\pi \in \{\frac{1}{4}, \frac{1}{2}, \frac{3}{4}\}$, $-\log \sigma_1 \in \{0, 1, 2\}$ and $-\log \sigma_2 \in \{6, 7, 8\}$.

\begin{figure}
\begin{center}
\includegraphics[width=0.9\columnwidth]{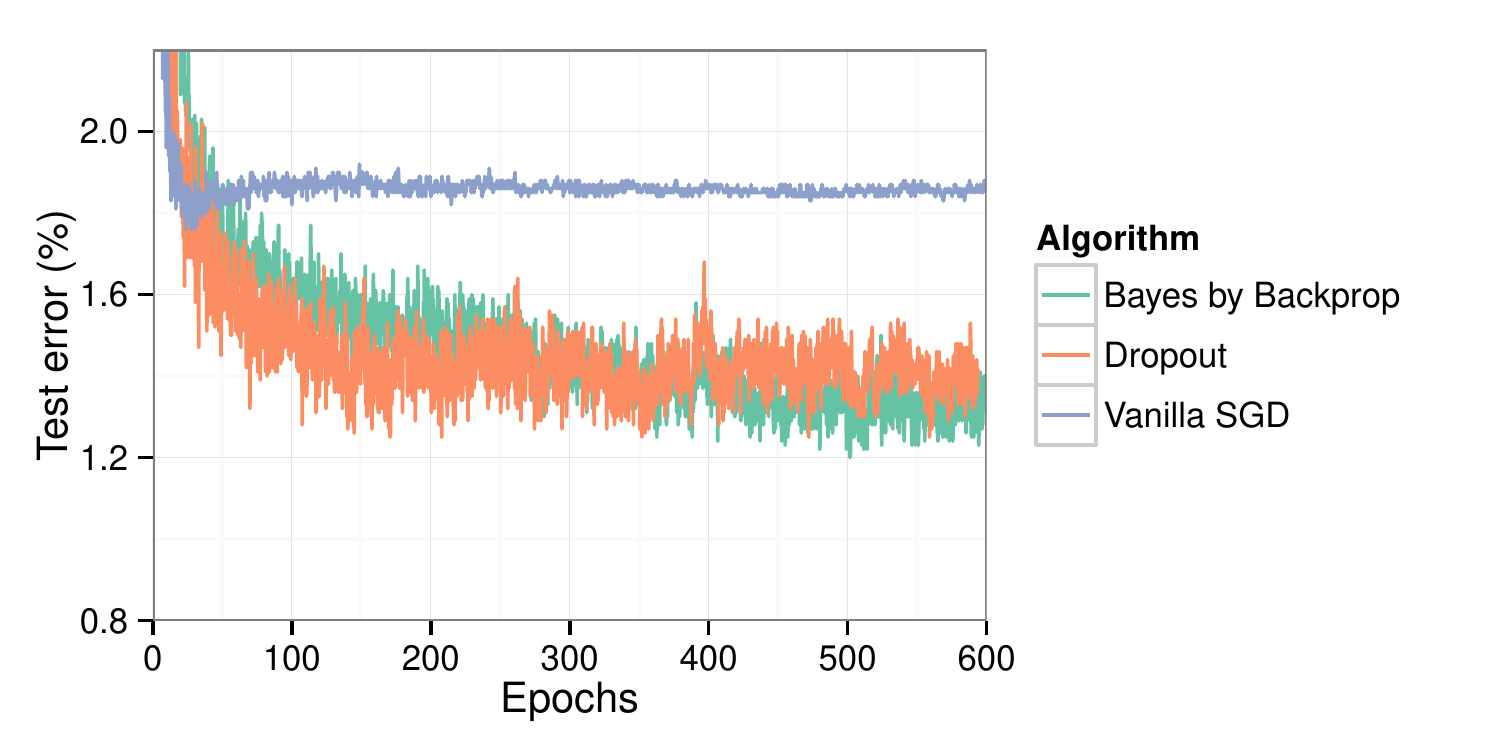}
\caption{Test error on MNIST as training progresses.}
\label{fig:mnist}
\end{center}
\end{figure}

\begin{figure}
\begin{center}
\includegraphics[width=0.9\columnwidth]{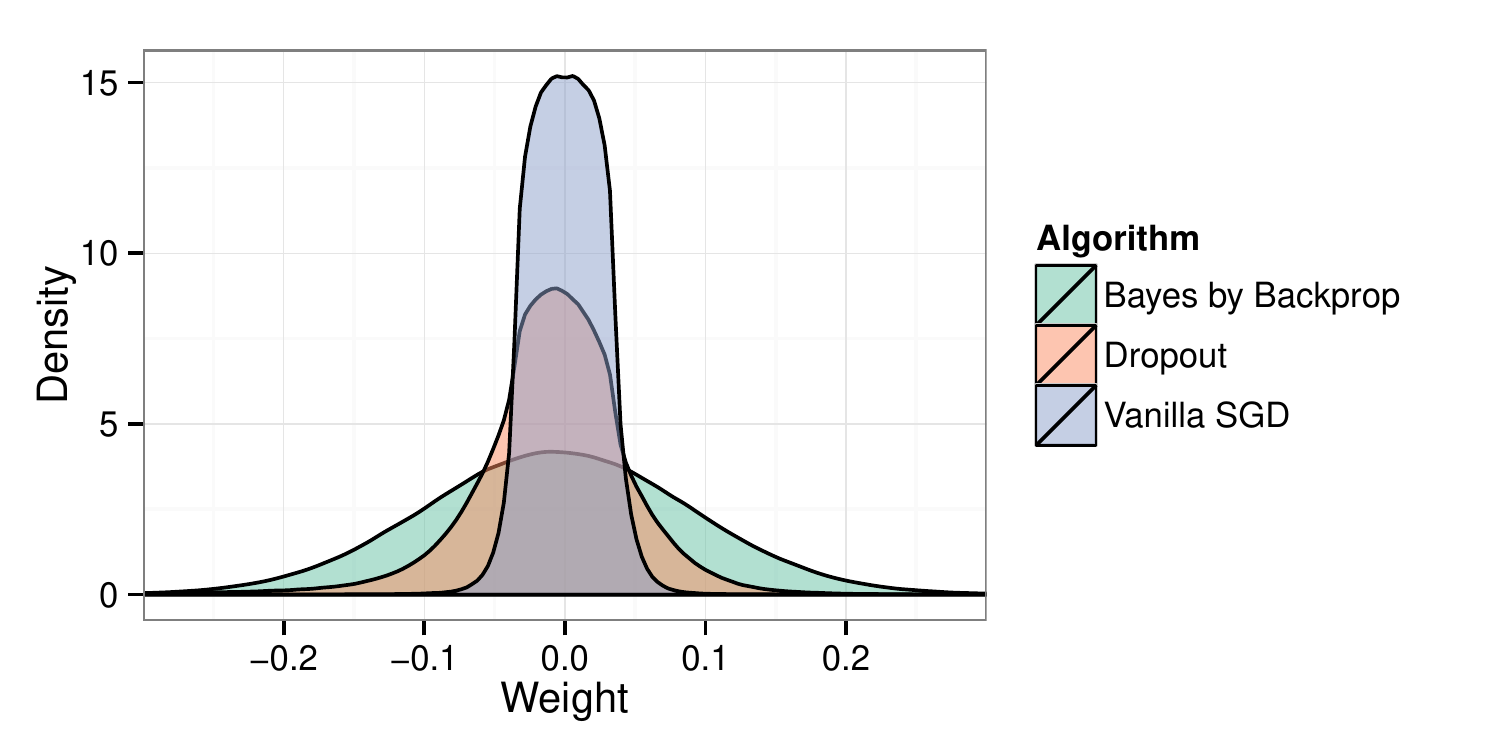}
\caption{Histogram of the trained weights of the neural network, for Dropout, plain SGD, and samples from Bayes by Backprop.}
\label{fig:histogram}
\end{center}
\end{figure}

Figure~\ref{fig:mnist} shows the learning curves on the test set for Bayes by
Backprop, dropout and SGD on a network with two layers of 1200 rectified linear
units.
As can be seen, SGD converges the quickest, initially obtaining a low test
error and then overfitting.
Bayes by Backprop and dropout converge at similar rates (although each
iteration of Bayes by Backprop is more expensive than dropout -- around two
times slower).
Eventually Bayes by Backprop converges on a better test error than dropout after
600 epochs.

Figure~\ref{fig:histogram} shows density estimates of the weights.  The Bayes
by Backprop weights are sampled from the variational posterior, and the dropout
weights are those used at test time.
Interestingly the regularised networks found by dropout and Bayes by Backprop
have a greater range and with fewer centred at zero than those found by SGD.
Bayes by Backprop uses the greatest range of weights.

\begin{figure}[h]
\begin{center}
\includegraphics[width=1\columnwidth]{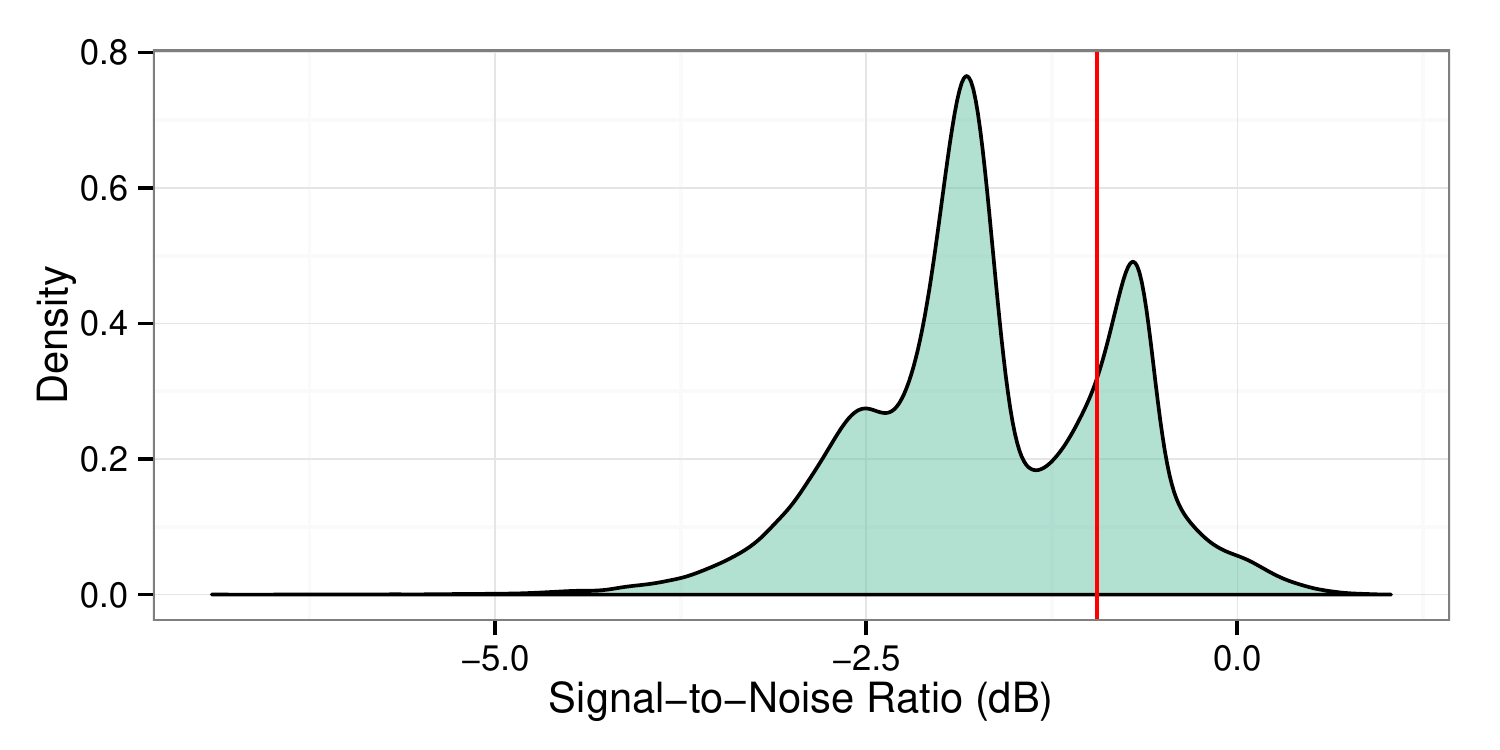}
\includegraphics[width=1\columnwidth]{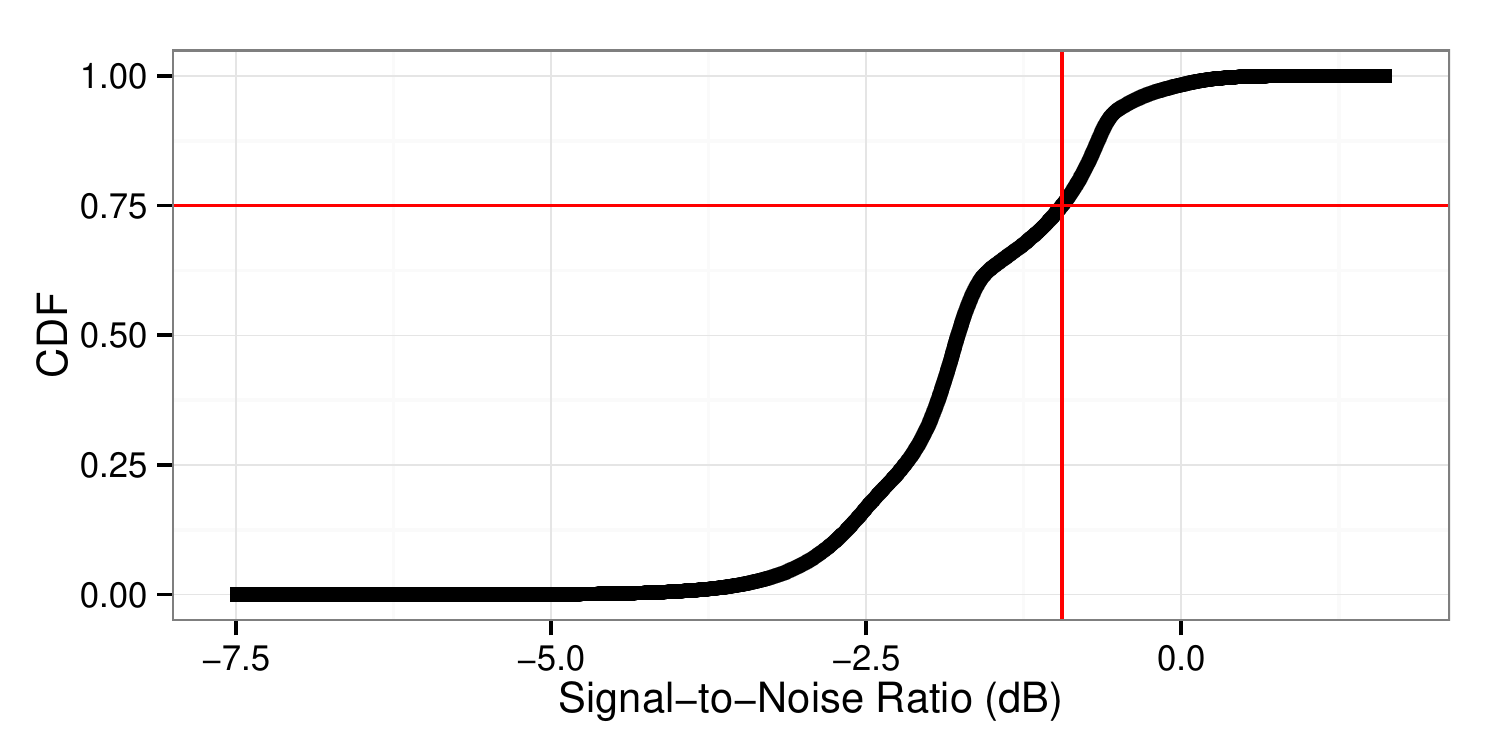}
\caption{Density and CDF of the Signal-to-Noise ratio over all weights in the
network. The red line denotes the 75\% cut-off.}
\label{fig:SNR}
\end{center}
\end{figure}

In Table~\ref{tab:snr}, we examine the effect of replacing the variational
posterior on some of the weights with a constant zero, so as
to determine the level of redundancy in the network found by Bayes by Backprop.
We took a Bayes by Backprop trained network with two layers of 1200
units\footnote{We used a network from the end of training rather
than picking a network with a low validation cost found during training, hence
the disparity with results in Table~\ref{tab:mnist}.  The lowest test error
observed was $1.12\%$.}
and ordered the weights by their signal-to-noise ratio ($|\mu_i|/\sigma_i$).
We removed the weights with the lowest signal to noise ratio.
As can be seen in Table~\ref{tab:snr}, even when $95\%$ of the weights are
removed the network still performs well, with a significant drop
in performance once $98\%$ of the weights have been removed.

In Figure~\ref{fig:SNR} we examined the distribution of the signal-to-noise
relative to the cut-off in the network uses in Table~\ref{tab:snr}.
The lower plot shows the cumulative distribution of signal-to-noise ratio, whilst
the top plot shows the density.
From the density plot we see there are two modalities of 
signal-to-noise ratios, and from the CDF we see that the 
75\% cut-off separates these two peaks.
These two peaks coincide with a drop in performance in Table~\ref{tab:snr}
from 1.24\% to 1.29\%, suggesting that the signal-to-noise heuristic
is in fact related to the test performance.

It is interesting to contrast this weight removal approach to obtaining a fast,
smaller, sparse network for prediction after training with the approach taken
by distillation \citep{hinton_distilling_2014} which requires an extra stage of
training to obtain a compressed prediction model.
\begin{table}
\caption{\label{tab:snr}Classification Errors after Weight pruning}
\centering
\begin{tabular}{l|l|r}
\textbf{Proportion removed} & \textbf{\# Weights} & \textbf{Test Error} \\
\hline
0\% & 2.4m &  $1.24\%$ \\
50\% & 1.2m & $1.24\%$ \\
75\% & 600k & $1.24\%$ \\
95\% & 120k & $1.29\%$ \\
98\% & 48k &  $1.39\%$ \\
\end{tabular}
\end{table}
As with distillation, our method begins with an ensemble (one for each
possible assignment of the weights).
However, unlike distillation, we can simply obtain a subset of this ensemble by
using the probabilistic properties of the weight distributions learnt to
gracefully prune the ensemble down into a smaller network.
Thus even though networks trained by Bayes by Backprop may have
twice as many weights, the number of parameters that actually
need to be stored at run time can be far fewer.
\citet{graves_practical_2011} also considered pruning weights using the signal
to noise ratio, but demonstrated results on a network $20$ times smaller and did
not prune as high a proportion of weights (at most $11\%$) whilst still
maintaining good test performance.
The scale mixture prior used by Bayes by Backprop encourages a broad spread of
the weights.
Many of these weights can be successfully pruned without impacting performance
significantly.

\subsection{Regression curves}

We generated training data from the curve:
\begin{align*}
y &= x + 0.3 \sin(2\pi (x + \epsilon)) + 0.3 \sin(4\pi (x + \epsilon)) + \epsilon
\end{align*}
where $\epsilon \sim \mathcal{N}(0, 0.02)$.
Figure \ref{fig:regress} shows two examples of fitting a neural network to these data,
minimising a conditional Gaussian loss.
Note that in the regions of the input space where there are no data, the
ordinary neural network reduces the variance to zero and chooses to fit a particular function,
even though there are many possible extrapolations of the training data.
\begin{figure}[bhtp]
\includegraphics[width=.49\linewidth]{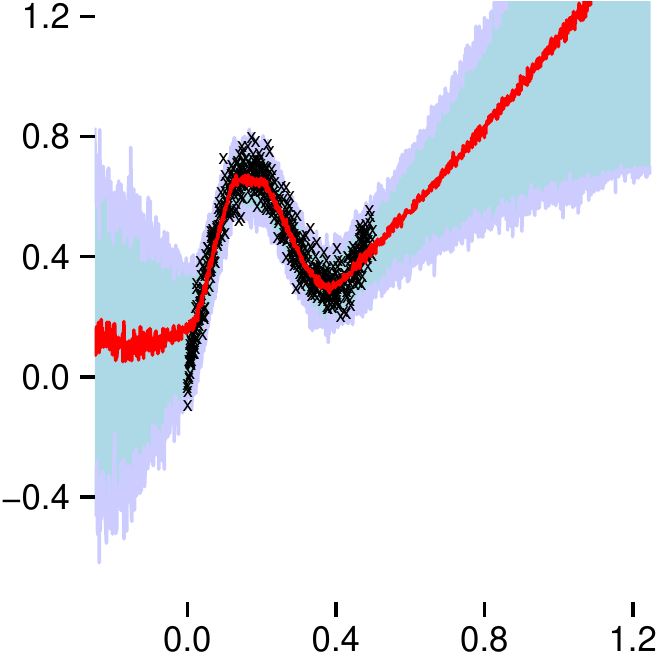}
\includegraphics[width=.49\linewidth]{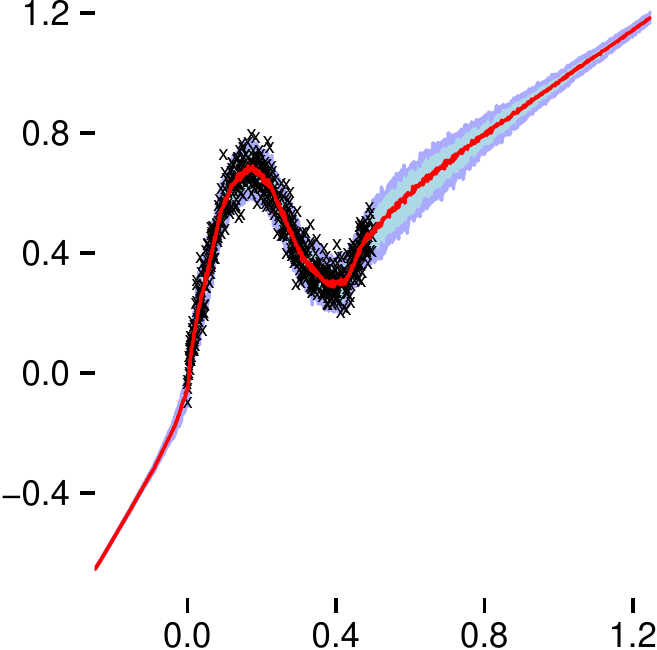}
\caption{\label{fig:regress}Regression of noisy data with interquatile ranges.
Black crosses are training samples. Red lines are median predictions. Blue/purple
region is interquartile range.
Left: Bayes by Backprop neural network, Right: standard neural network.}
\end{figure}
On the left, Bayesian model averaging affects predictions:
where there are no data, the confidence intervals diverge, reflecting there being many possible extrapolations.
In this case Bayes by Backprop prefers to be uncertain where there are no
nearby data, as opposed to a standard neural network which can be overly
confident.

\subsection{Bandits on Mushroom Task}
\label{sec:exper:bandits}
We take the UCI Mushrooms data set \citep{bache_uci_2013}, and cast it as a bandit task, similar to
\citet[Chapter 6]{guez_sample_2015}. 
Each mushroom has a set of features, which we treat as the context for the
bandit, and is labelled as edible or poisonous.
An agent can either eat or not eat a mushroom.
If an agent eats an edible mushroom, then it receives a reward of $5$.
If an agent eats a poisonous mushroom, then with probability $\frac{1}{2}$ it receives a reward of $-35$, otherwise
a reward of $5$.
If an agent elects not to eat a mushroom, it receives a reward of $0$.
Thus an agent expects to receive a reward of $5$ for eating an edible reward,
but an expected reward of $-15$ for eating a poisonous mushroom.

Regret measures the difference between the reward achievable by an oracle and the reward received by an agent.
In this case, an oracle will always receive a reward of $5$ for an edible mushroom, or $0$ for a poisonous mushroom.
We take the cumulative sum of regret of several agents and show them in Figure~\ref{fig:mushroom}.
\begin{figure}
\includegraphics[width=\linewidth]{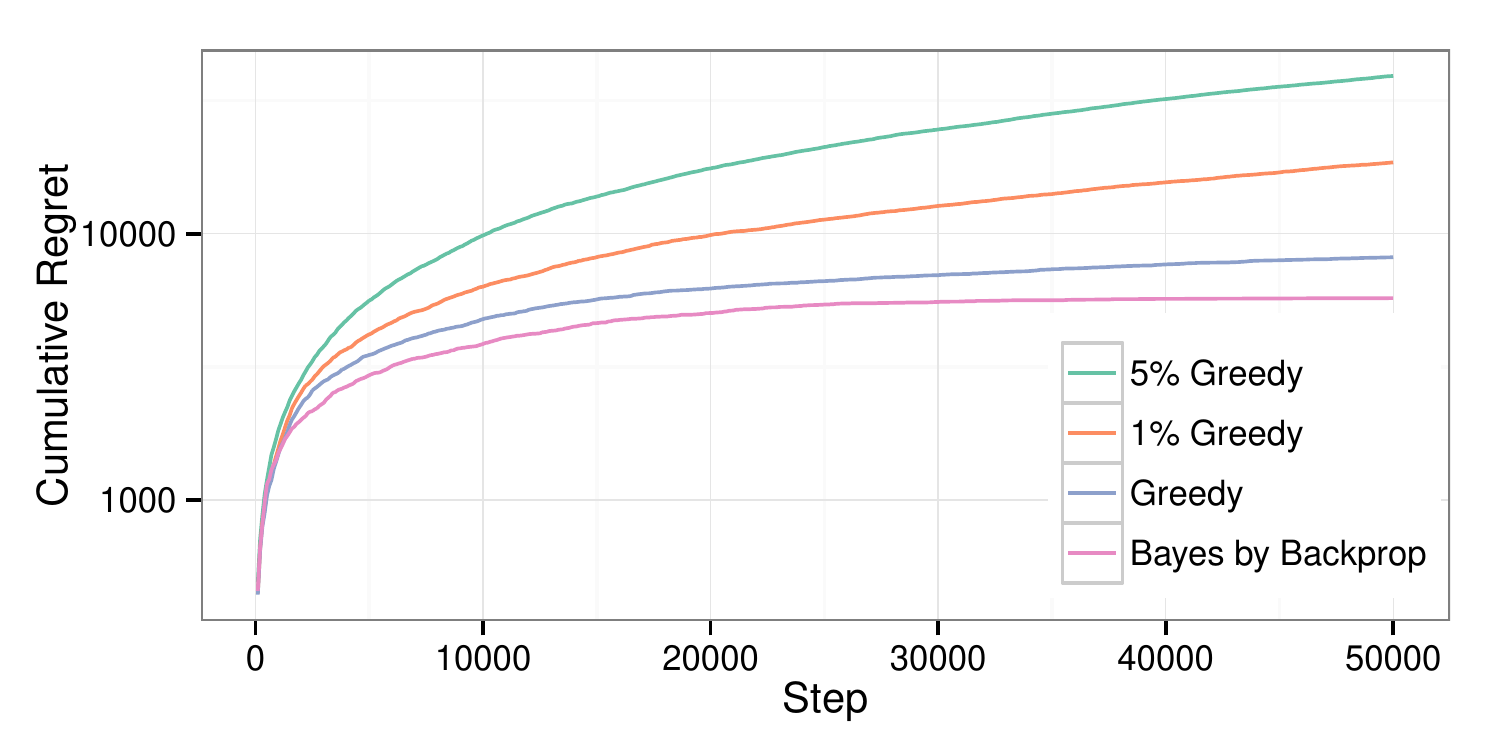}
\caption{\label{fig:mushroom}Comparison of  cumulative regret of various agents on the mushroom bandit task, averaged over five runs. Lower is better.}
\end{figure}
Each agent uses a neural network with two hidden layers of 100 rectified linear units.
The input to the network is a vector consisting of the mushroom features (context) and a one of $K$ encoding of the
action.
The output of the network is a single scalar, representing the expected reward of the given action in the given context.
For Bayes by Backprop, we sampled the weights twice and averaged two of these outputs to obtain the
expected reward for action selection.
We kept the last 4096 reward, context and action tuples in a buffer, and
trained the networks using randomly drawn minibatches of size 64 for 64
training steps ($64\times 64 = 4096$) per interaction with the Mushroom bandit.
A common heuristic for trading-off exploration vs. exploitation is to follow an
$\varepsilon$-greedy policy: with probability $\varepsilon$ propose a uniformly random action,
otherwise pick the best action according to the neural network.

Figure~\ref{fig:mushroom} compares a Bayes by Backprop agent with
three $\varepsilon$-greedy agents, for values of $\varepsilon$ of $0\%$ (pure
greedy), $1\%$, and $5\%$.
An $\varepsilon$ of $5\%$ appears to over-explore, whereas a purely greedy agent does poorly
at the beginning, greedily electing to eat nothing, but then does much better
once it has seen enough data.
It seems that non-local function approximation updates allow the greedy agent
to explore, as for the first $1,000$ steps, the agent eats nothing but after
approximately $1,000$ the greedy agent suddenly decides to eat mushrooms.
The Bayes by Backprop agent explores from the beginning, both eating and ignoring
mushrooms and quickly converges on eating and non-eating with an almost perfect
rate (hence the almost flat regret).

\section{Discussion}
\label{sec:discuss}

We introduced a new algorithm for learning neural networks with uncertainty on
the weights called Bayes by Backprop.
It optimises a well-defined objective function to learn a distribution on the
weights of a neural network.
The algorithm achieves good results in several domains.
When classifying MNIST digits, performance from Bayes by Backprop is comparable
to that of dropout.
We demonstrated on a simple non-linear regression problem that the uncertainty
introduced allows the network to make more reasonable predictions about unseen
data.
Finally, for contextual bandits, we showed how Bayes by Backprop can
automatically learn how to trade-off exploration and exploitation.
Since Bayes by Backprop simply uses gradient updates, it can readily be scaled
using multi-machine optimisation schemes such as asynchronous SGD \citep{dean_large_2012}.
Furthermore, all of the operations used are readily implemented on a GPU.

\paragraph{Acknowledgements}

The authors would like to thank
Ivo Danihelka,
Danilo Rezende,
Silvia Chiappa,
Alex Graves,
Remi Munos,
Ben Coppin,
Liam Clancy,
James Kirkpatrick,
Shakir Mohamed,
David Pfau, and
Theophane Weber
for useful discussions and comments.

\bibliographystyle{plainnat}
\bibliography{refs.bib}

\end{document}